\documentclass[letterpaper, 10 pt, conference]{ieeeconf} 
\IEEEoverridecommandlockouts                              
\usepackage[utf8]{inputenc}
\usepackage[T1]{fontenc}
\usepackage{amsmath,bm}
\usepackage{amsfonts}
\usepackage{algorithm}
\usepackage{algorithmic}
\usepackage{amsthm}
\let\labelindent\relax  
\usepackage{enumitem}
\usepackage[english]{babel}
\usepackage{cite}

\usepackage[colorlinks=false, citebordercolor={1 1 1}, linkbordercolor={1 1 1}]{hyperref}
\usepackage{threeparttable} 
\usepackage{multirow}
\usepackage{graphicx}
\usepackage[font=footnotesize]{caption}
\usepackage{subcaption}
\usepackage{titlesec}
\usepackage{nth}
\usepackage{cleveref} 
\usepackage{lipsum}  
\usepackage[dvipsnames]{xcolor}
\usepackage{colortbl}
\definecolor{light-gray}{gray}{0.93}
\definecolor{semilight-gray}{gray}{0.8}

\usepackage{amssymb}  

\newtheoremstyle{defiStyle}  
  {3pt} {3pt} {\normalfont} {} {\bfseries\itshape} {:} { }  
  {\thmname{#1} \thmnumber{#2}\textit{\thmnote{ (#3)}}}  

\theoremstyle{defiStyle}  
\newtheorem{defi}{Definition}  

\theoremstyle{thmStyle}
\newtheorem{thm}{Theorem}  
\newtheorem{remark}{Remark}  
\newtheorem{prob}{Problem}

\title{\LARGE \bf
Long-Horizon Geometry-Aware Navigation among Polytopes \\ via MILP-MPC and Minkowski-Based CBFs
}

\author{Yi-Hsuan Chen$^{1}$, Salman Ghori$^{2}$, Ania Adil$^{2}$, Eric Feron$^{2}$, Calin Belta$^{3}$ 
\thanks{$^{1}$Y. Chen is with the Dept. of Aerospace Engineering, University of Maryland, College Park, MD, USA. {\tt yhchen91@umd.edu}}%
\thanks{$^{2}$S. Ghori, A. Adil, and E. Feron are with Computer, Electrical and Mathematical Science \& Engineering Division (CEMSE), KAUST, Saudi Arabia. {\tt \{salman.ghori, ania.adil, eric.feron\}@kaust.edu.sa}}%
\thanks{$^{3}$C. Belta is with the Dept. Electrical and Computer Engineering and Dept. Computer Science, University of Maryland, College Park, MD, USA. {\tt calin@umd.edu}}
}

\newcommand{\mdspace}{$\mathcal{M}_\mathcal{D}$-space}
\newcommand{\cobs}{CO}
\newcommand{\Aci}{\mathbf{A}_\mathcal{I}^c}
\newcommand{\lambdai}{\boldsymbol{\lambda}_{\mathcal{I}}}

\begin{document}
\maketitle
\begin{abstract}  
Autonomous navigation in complex, non-convex environments remains challenging when robot dynamics, control limits, and exact robot geometry must all be taken into account. In this paper, we propose a hierarchical planning and control framework that bridges long-horizon guidance and geometry-aware safety guarantees for a {\it polytopic} robot navigating among {\it polytopic} obstacles. At the high level, Mixed-Integer Linear Programming (MILP) is embedded within a Model Predictive Control (MPC) framework to generate a nominal trajectory around polytopic obstacles while modeling the robot as a point mass for computational tractability. At the low level, we employ a control barrier function (CBF) based on the exact signed distance in the Minkowski-difference space as a safety filter to explicitly enforce the geometric constraints of the robot shape, and further extend its formulation to a high-order CBF (HOCBF). We demonstrate the proposed framework in U-shaped and maze-like environments under single- and double-integrator dynamics. The results show that the proposed architecture mitigates the topology-induced local-minimum behavior of purely reactive CBF-based navigation while enabling safe, real-time, geometry-aware navigation. A video of all simulations can be found at \textnormal{\url{https://youtu.be/-cm4cQ_WXDY}}.

\end{abstract}

\section{Introduction}
\begin{figure}[ht!]
	\centering
    \includegraphics[width=0.932\linewidth]{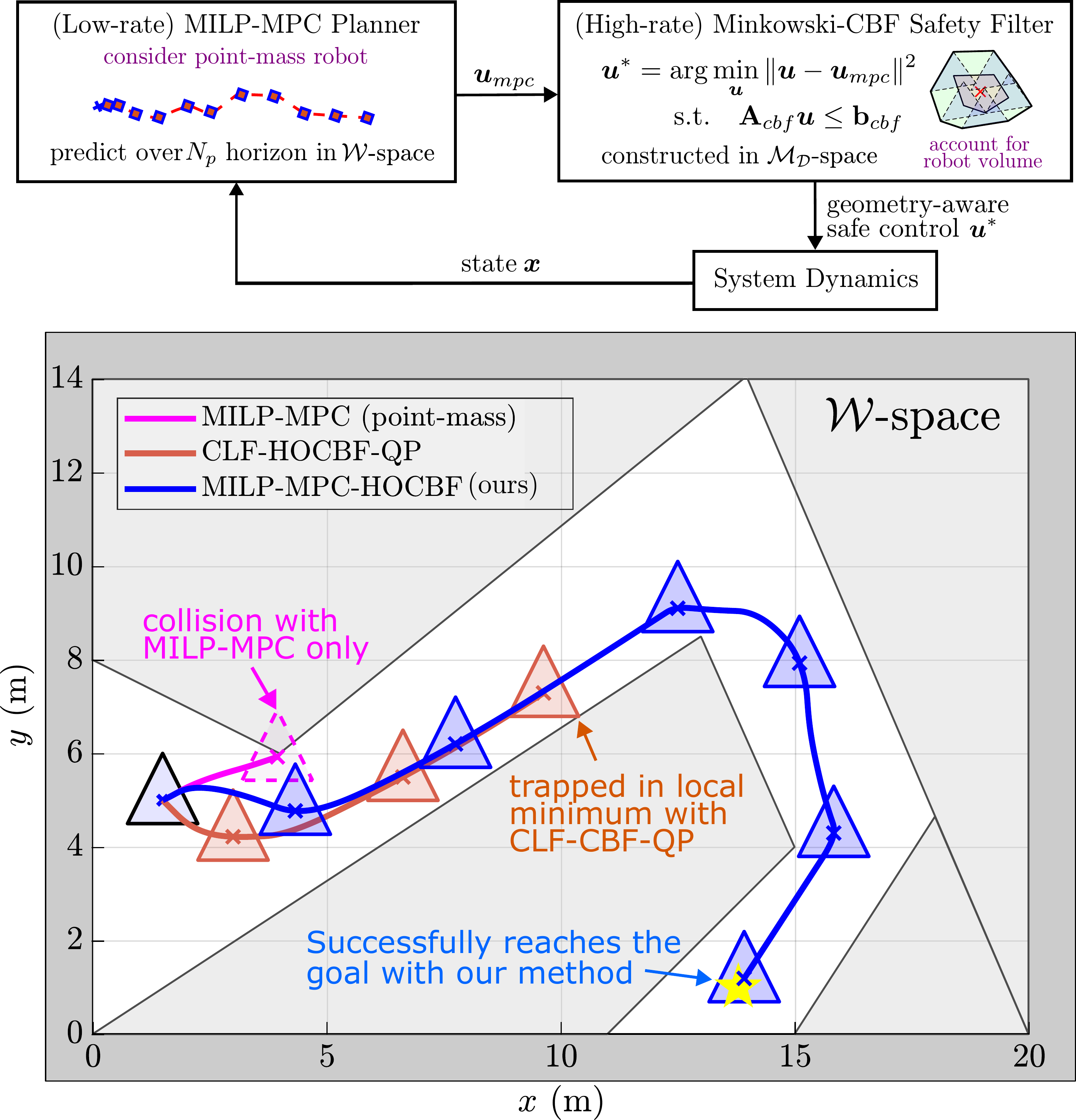}
	\caption{\textbf{Top:} The proposed multi-rate framework with a MILP-MPC planner and a geometry-aware Minkowski-CBF safety filter. \textbf{Bottom:} A triangular robot with double-integrator dynamics navigating in a polytopic maze-like environment. It collides with an obstacle when using MILP-MPC alone with a point-mass model (pink), and becomes trapped in a local minimum using a CLF-CBF-QP approach (orange) due to its inherent reactive nature. In contrast, the proposed framework (blue) successfully guides the robot to the goal {\it without a reference path or waypoints from an external planner}.}\label{fig: teaser}
\end{figure}

Autonomous navigation in cluttered, non-convex environments is a challenging problem when robot dynamics, control limits, and robot-obstacle geometry must all be considered simultaneously.
To make collision-avoidance constraints explicit and computationally tractable, many constrained planning and control methods approximate the robot as a point mass or a sphere, which we refer to as simplified geometric models \cite{Latombe.Planning.12,LaValle.Planning.06}. While computationally efficient, these approximations can be conservative and may rule out feasible motions that exist for the true robot geometry in tight environments.

To reduce this conservatism, recent work has considered tighter geometric representations such as polytopes. Optimization-based approaches have long considered polytopic constraints, particularly through mixed-integer formulations. Mixed-Integer Linear Programming (MILP) has been widely used to encode such constraints through binary variables and disjunctive formulations explicitly \cite{Schouwenaars.etal.ECC01,Richards.How.ACC02,Vitus.etal.GNC08,Lin.etal.TITS11,Deits.Tedrake.ICRA15}. More recently, an advanced framework of graphs of convex sets \cite{Marcucci.etal.SR23} provides more efficient formulations. However, for tractability, these methods are typically restricted to simplified geometric models. Alternatively, dualization techniques \cite{Zhang.OBCA.TCST20} have been proposed to handle both robot and obstacle geometry, but often lead to nonlinear programs that are computationally demanding and rely on good warm-starts. 

Control Barrier Functions (CBFs) \cite{Ames.etal.ECC19} and their generalization for high-order systems \cite{ Xiao.Calin.TAC21} offer
an efficient alternative for control-affine systems by rendering a safe set forward invariant through a sequence of pointwise Quadratic Programs (QPs) that can be solved on-the-fly. They have recently been extended to polytopic collision avoidance using both approximate \cite{Molnar.CCTA25,Usevitch.Sahleen.ACC25, Wu.etal.TCB25, Fernandez.etal.CASE25} and exact geometric \cite{Chen.etal.CDC25, Thirugnanam.etal.ACC22, Matias.Silvestre.arXiv26} representations. For reach-avoid tasks, CBFs are often paired with Control Lyapunov Functions (CLF), yielding the canonical CLF-CBF-QP framework \cite{Ames.etal.ECC19}. However, without waypoint guidance or a feasible reference path, such methods remain inherently reactive and may become trapped in local minima \cite{Reis.etal.LCSS20, Tan.Dimarogonas.Auto24}, especially in non-convex environments.


To mitigate this myopic behavior, we present a hierarchical navigation framework that combines a {\it long-horizon} planner with a {\it geometry-aware} safety filter. At the high level, a MILP is embedded within {an} MPC framework to generate a long-horizon nominal trajectory around non-convex polytopic obstacles, while using {\it simplified} geometric models for tractability. At the low level, a CBF based on the signed distance between a polytopic robot-obstacle pair in the Minkowski-difference space (\mdspace), termed Minkowski-CBF \cite{Chen.etal.CDC25}, serves as a safety filter that enforces the {\it exact} geometric constraints. The overall architecture is inspired by the design principles of multi-rate control \cite{Rosolia.Ames.LCSS21, Garg.etal.LCSS21,Matni.etal.CSM24} and run-time assurance \cite{Aiello.etal.10JGCD,Ghori.etal.DASC22,Hobbs.etal.CSM23}. 

Prior work has combined MPC and discrete CBFs (DCBFs) to integrate long-horizon prediction with safety-critical control \cite{Grandia.etal.ICRA21 ,Zeng.etal.ACC21, Liu.etal.ACC23}. Similar ideas have also been explored by combining Model Predictive Path Integral (MPPI) control with CBFs \cite{Yin.etal.RAL23, Rabiee.Hoagg.CDC25}. While most of these formulations do not explicitly account for robot volume, recent extensions to obstacle avoidance among polytopes are more closely related to our setting \cite{Thirugnanam.etal.ICRA22,Liu.etal.arXiv26}. Our work differs from \cite{Thirugnanam.etal.ICRA22,Liu.etal.arXiv26} in two key aspects: (1) Rather than embedding DCBF constraints directly into the MPC optimization problem, our method explicitly separates long-horizon planning from high-rate safety filtering. (2) Instead of relying on a reference path or waypoint sequence provided by an external global planner, our method uses the MILP-MPC layer itself to generate long-horizon guidance. Another recent work \cite{Yamaguchi.etal.arXiv26} also considers a layered CBF architecture that combines occupancy-map-based Poisson safety functions with a predictive safety layer and a real-time CBF-QP filter. In contrast, our framework is tailored to {\it exact-geometry navigation with polytopic sets}, combining an MILP-MPC planner with a Minkowski-CBF safety filter to reduce conservatism induced by simplified geometric models.

The contributions of this paper are threefold:
\begin{itemize}
    \item We propose a hierarchical multi-rate navigation framework, referred to as MILP-MPC-CBF, that combines a low-rate MILP-MPC planner with a high-rate geometry-aware  Minkowski-CBF safety filter.
    \item We derive a closed-form gradient expression of the Minkowski-CBF \cite{Chen.etal.CDC25} in the minimum-distance, pure-translation setting by exploiting the KKT structure of the underlying min-dist QP, and further extend the formulation to a HOCBF for double-integrator systems.
    \item We validate the proposed MILP-MPC-CBF framework in non-convex U-shaped and maze-like environments under both single- and double-integrator dynamics. Comparisons with the canonical CLF-CBF-QP approach show that the latter becomes trapped, whereas the proposed framework safely reaches the goal.
\end{itemize}
\section{Preliminaries}
In this section, we briefly review the background relevant to our Minkowski-CBF formulation, including configuration obstacles (\ref{subsec:prelim_CO}), HOCBF (\ref{subsec:prelim_cbf}), and differentiable optimization (\ref{subsec:prelim_diffopt}).

\subsection{Configuration Obstacles via Minkowski Operation}\label{subsec:prelim_CO}
Originating from computational geometry, Minkowski operations have been a fundamental tool for collision detection \cite{gjk,epa}, and were later adopted in motion planning to define the so-called ``Configuration Obstacle'' (see Def.~\ref{def:CO}) \cite{Lozano.TC83,Ericson}. By taking the Minkowski difference between an obstacle and the robot, their associated CO encodes the robot’s geometry and configuration with the obstacle in a single set. Consequently, the collision checking between two sets is reduced to a set membership problem.
\begin{defi} [\it Configuration Obstacle (CO) \cite{LaValle.Planning.06}] \label{def:CO}
\normalfont
    Given a robot $\mathcal{R}$ and an obstacle $\mathcal{O}$, the associated \cobs~is:
    \begin{align}
        \mathcal{O}^c = \mathcal{O}\oplus (-\mathcal{R})\label{eq: co},
    \end{align}
    where $-\mathcal{R}$ is the reflection of $\mathcal{R}$ through the origin and $\oplus$ denotes the Minkowski sum. 
\end{defi}

\subsection{High-Order Control Barrier Functions (HOCBF)}\label{subsec:prelim_cbf}
Consider a control-affine system 
:
\begin{align}
    \dot{\bm{x}}=f(\bm{x})+g(\bm{x})\bm{u},\label{eq: ctrl-affine}
\end{align}
where $\bm{x}\in \mathbb{R}^n$ is the state, $\bm{u}\in \mathcal{U}\subset\mathbb{R}^q$ is the input ($\mathcal{U}$ denotes a closed control constraint set), and $f:\mathbb{R}^n\rightarrow \mathbb{R}^n$, $g:\mathbb{R}^n\rightarrow \mathbb{R}^{n\times q}$ are locally Lipschitz.

\begin{defi}[\it Forward invariant set \cite{Ames.etal.ECC19}]
    \normalfont A set $C\subset\mathbb{R}^n$ is forward invariant for system \eqref{eq: ctrl-affine} if its solutions starting at $\bm{x}(t_0)\in C$ satisfy $\bm{x}(t)\in C,\,\,\forall t\geq t_0$.
\end{defi}


In this article, safety is guaranteed by the constraint $h(\bm{x})\geq0$; we refer to the relative degree\footnote{The relative degree of a differentiable function $h(\bm{x}):\mathbb{R}^n\rightarrow \mathbb{R}$ with respect to system \eqref{eq: ctrl-affine} is the number of times we have to differentiate it along the dynamics of \eqref{eq: ctrl-affine} until control $\bm{u}$ appears explicitly \cite{Khalil.nonlinear}.}
of $h$ as the relative degree of the constraint. For a function $h(\bm{x}):\mathbb{R}^n\rightarrow \mathbb{R}$ with relative degree $m$, we define a sequence of functions $\psi_i:\mathbb{R}^n\rightarrow\mathbb{R}$, $i\in\{1,\cdots,m\}$ as
\begin{align}
    \psi_i(\bm{x})=\dot{\psi}_{i-1}(\bm{x}) + \alpha
    _i(\psi_{i-1}(\bm{x})),\quad i\in\{1,\cdots,m\},\label{eq: psi_hocbf}
\end{align}
where each $\alpha_i$ is a class-$\mathcal{K}$ function\footnote{A class $\mathcal{K}$ function is a Lipschitz continuous function $\alpha:[0,a)\rightarrow [0,\infty),\, a>0$ that is strictly increasing and $\alpha(0)=0$, see, e.g., \cite{Khalil.nonlinear}.} and $\psi_0(\bm{x}):=h(\bm{x})$. We then define a sequence of sets $C_i$ associated with \eqref{eq: psi_hocbf}:
\begin{align}
    C_i=\{\bm{x} \in \mathbb{R}^n: \psi_{i-1}(\bm{x})\geq0 \},\quad i\in\{1,\cdots,m\}.\label{eq: Cs_hocbf}
\end{align}

\begin{defi}[\it HOCBF \cite{Xiao.Calin.TAC21}]\label{def: hocbf}
    \normalfont Let $C_1,\cdots,C_m$ be defined in \eqref{eq: Cs_hocbf} and $\psi_1,\cdots,\psi_m$ be defined in \eqref{eq: psi_hocbf}. A function $h:\mathbb{R}^n\rightarrow\mathbb{R}$ is a High-Order Control Barrier Function (HOCBF) of relative degree $m$ for system \eqref{eq: ctrl-affine} if there exist differentiable class-$\mathcal{K}$ functions $\alpha_i,\,\,i\in\{1,\cdots,m\}$ such that
    \begin{align}\begin{split}
    \textstyle\sup_{\bm{u}\in\mathcal{U}}      [L_f^mh(\bm{x})+L_gL_f^{m-1}h(\bm{x})\bm{u}+O(h(\bm{x}))\\+\alpha_m(\psi_{m-1}(\bm{x}))]\geq0,\quad \forall\bm{x}\in C_1\cap\cdots\cap C_m,
    \label{eq: def_hocbf_constr}
    \end{split}\end{align}
    where $L_f, L_g$ denote the Lie derivatives along $f$ and $g$, respectively, $L_f^m$ denotes $m$-th Lie derivative along $f$ and $O(h(\bm{x}))=\sum_{i=1}^{m-1} L_f^i(\alpha_{m-i} \circ \psi_{m-i-1})(\bm{x})$. It is assumed that $L_gL_f^{m-1}h(\bm{x})\neq 0$ on the boundary of the set $ C_1\cap\cdots\cap C_m$.
    
\end{defi}


\begin{thm}[Safety Guarantee \cite{Xiao.Calin.TAC21}]
    \normalfont Given a HOCBF $h(\bm{x})$ defined by \eqref{def: hocbf} with the associated sets $C_1,\cdots,C_m$ defined by \eqref{eq: Cs_hocbf}, if $\bm{x}(0)\in C_1\cap\cdots\cap C_m$, then any Lipschitz continuous controller $\bm{u}(t)\in \mathcal{U}$ that satisfies \eqref{eq: def_hocbf_constr}, $\forall t\geq 0$ renders $C_1\cap\cdots\cap C_m$ forward invariant for system \eqref{eq: ctrl-affine}.
\end{thm}

\subsection{Differentiable Optimization}\label{subsec:prelim_diffopt}
Recently, differentiable optimization has been increasingly adopted in robotics for collision avoidance, to compute the gradients of CBFs defined via optimization \cite{Dai.etal.RAL23, Chen.etal.CDC25, Fernandez.etal.RAL26}. Consider a parameteric inequality-constrained QP:
\begin{align}\begin{split}
    &\mathbf{z}^*=\arg\min_{\mathbf{z}} f_o(\mathbf{z}),\quad\text{s.t.}\,\,\, \mathbf{G}(\mathbf{x})\mathbf{z}\leq \mathbf{h}(\mathbf{x}),\\
    &\text{with} \,\,\,\,f_o(\mathbf{z})=\mathbf{z}^\top \mathbf{Q}\mathbf{z}+\mathbf{p}^\top\mathbf{z},\quad\mathbf{Q}\in\mathbb{S}_{++}^k,\,\,\mathbf{p}\in\mathbb{R}^k
    \label{eq: diff-qp}
\end{split}\end{align}
where $\mathbf{z}\in\mathbb{R}^k$ is the optimization variable, $\mathbf{x}\in\mathbb{R}^p$ is the parameter, 
and $\mathbf{G}(\mathbf{x})\in\mathbb{R}^{\mathfrak{g}\times k}$, $\mathbf{h}(\mathbf{x})\in\mathbb{R}^\mathfrak{g}$ define the parameter-dependent 
inequality constraints. The Lagrangian of \eqref{eq: diff-qp} is
\begin{align}
    \mathcal{L}(\mathbf{z},\boldsymbol{\lambda},\mathbf{x}):= f_o(\mathbf{z})+ \boldsymbol{\lambda}^\top (\mathbf{G}(\mathbf{x})\mathbf{z}-\mathbf{h}(\mathbf{x})), \label{eq: L}
\end{align}
where $\boldsymbol{\lambda}\in\mathbb{R}^\mathfrak{g}$ is the dual variable. The optimal primal-dual pair $(\mathbf{z}^*(\mathbf{x}), \boldsymbol{\lambda}^*(\mathbf{x}))$ satisfies the following compact Karush-Kuhn-Tucker (KKT) system\footnote{Assuming linear independence constraint qualification (LICQ) holds.
}: 
\begin{align}
    \Gamma(\mathbf{z},\boldsymbol{\lambda},\mathbf{x}):=\begin{bmatrix}
    2\mathbf{Q}\mathbf{z}+\mathbf{p} + \mathbf{G}(\mathbf{x})^\top\boldsymbol{\lambda}
 \\D(\boldsymbol{\lambda})(\mathbf{G}(\mathbf{x})\mathbf{z}-\mathbf{h}(\mathbf{x}))
\end{bmatrix}=\bm{0},
\label{eq: KKT_cond}
\end{align}
where the first and second rows correspond to the stationarity and complementary slackness conditions, respectively. Here, $D(\cdot)$ creates a diagonal matrix from a vector; for brevity, we omit the dependence on $\mathbf{x}$ when clear from the context.
The sensitivity of the optimal value of the objective function $f_0(\mathbf{z}^*)$ is characterized by Theorem.~\ref{Thm:dfdx} from \cite{Castillo.etal.EngOpt6}.
We will later use this result to derive the gradient of our Minkowski-CBF (detailed in Sec.~\ref{sec: M-cbf}), as in \cite{Fernandez.etal.RAL26}.

\begin{thm}[Sensitivity of the optimal value of the objective function with respect to the parameter \cite{Castillo.etal.EngOpt6}]
    \normalfont The sensitivity of the optimal value of the objective function $f_0(\mathbf{z}^*)$ of \eqref{eq: diff-qp} with respect to the parameter $\mathbf{x}$ is given by the partial derivative of its Lagrangian function \eqref{eq: L} with respect to $\mathbf{x}$, evaluated at the optimal primal-dual pair $(\mathbf{z}^*, \boldsymbol{\lambda}^*)$, i.e., 
    $\nabla_{\mathbf{x}}f_0(\mathbf{z}^*(\mathbf{x}))= \frac{\partial \mathcal{L} }{\partial \mathbf{x}}\big\vert_{(\mathbf{z}^*,\boldsymbol{\lambda}^*)}$.
    \label{Thm:dfdx}
\end{thm}

\section{Problem Formulation}
In this paper, $\bm{x}$ denotes the robot state, and $\bm{p}\in\mathbb{R}^d$ denotes its position component. 
We adopt the same geometric assumptions for the robot, obstacles, and the associated COs as in \cite{Chen.etal.CDC25}, and recall only the half-space representations needed here. In particular, we assume that the robot $\mathcal{R}$ is convex and that each obstacle $\mathcal{O}_i$ is either convex or decomposed into convex components. Its corresponding CO, denoted as $\mathcal{O}_i^c$, is then used
for our pairwise CBF construction.
\begin{subequations}
    \begin{align}
        &\mathcal{R}(\bm{x})= \{\mathbf{y}\in \mathbb{R}^{d}\mid A_r(\bm{x}) \mathbf{y}\leq b_r(\bm{x})\}, \\
        &\mathcal{O}_i= \{\mathbf{y}\in \mathbb{R}^{d}\mid A_{o_i} \mathbf{y}\leq b_{o_i}\},\quad i=1,\ldots,N_\mathcal{O}
        \label{eq: obs-w}\\
        &\mathcal{O}_i^c(\bm{x})=\{\mathbf{y}\in \mathbb{R}^{d}\mid \mathbf{A}^c_i(\bm{x}) \mathbf{y}\leq \mathbf{b}^c_i(\bm{x})\},\label{eq: obs-md}
    \end{align}
\end{subequations}
with $A_r\in\mathbb{R}^{\ell_r\times d},~b_r\in\mathbb{R}^{\ell_r}$,  $A_{o_i}\in\mathbb{R}^{\ell_{o_i}\times d},~b_{o_i}\in\mathbb{R}^{\ell_{o_i}}$, $\mathbf{A}_i^c\in\mathbb{R}^{\ell_{c_i}\times d},~\mathbf{b}^c_i\in\mathbb{R}^{\ell_{c_i}}$. $N_\mathcal{O}$ is the number of convex obstacle components, and $d\in\{2,3\}$ for 2D and 3D spaces. $\ell_r$, $\ell_{o_i}$ and $\ell_{c_i}$ denote the number of edges of the robot and the $i$-th obstacle and CO, respectively. We suppress the dependence on $\bm{x}$ for brevity.

In general, a CO depends on both the robot's position and orientation. However, since our high-level planner uses an MILP formulation for linear systems only, the resulting robot's motion is purely translational; therefore, the robot orientation remains fixed, and so do the corresponding COs.
 
\begin{prob}
    Consider a workspace $\mathcal{X}$, obstacle space $\mathcal{X}_\mathcal{O}= \bigcup_{i=1}^{N_\mathcal{O}}\mathcal{O}_i\subseteq\mathcal{X}$ 
    , free space  $\mathcal{X}_\text{free}=\mathcal{X}\setminus \mathcal{X}_\mathcal{O}$, and a polytopic robot $\mathcal{R}$ with dynamics $\dot{\bm{x}}=A\bm{x}+B\bm{u}$. Our goal is to generate a feedback control law $\bm{u}^*(\bm{x})$ driving the robot from a given initial state $\bm{x}_s$ to a goal point $\bm{x}_g\in\mathcal{X}_\text{free}$, while ensuring $\mathcal{R}(\bm{x}(t))\cap\mathcal{X}_\mathcal{O}=\emptyset$ for all $t\geq0$.
\end{prob}

\begin{remark}
    Applying the MILP-MPC planner to nonlinear systems would require linearizing along a nominal trajectory. Since the planner operates at a low rate, the resulting linearization error must be treated carefully over the planning horizon, which is beyond the scope of the present work.
\end{remark}
Next we present the proposed hierarchical multi-rate planning architecture, which embeds a high-level MILP-MPC planner for long-horizon guidance with a low-level Minkowski-CBF safety filter for geometry-aware collision avoidance without geometric approximations.

\section{Multi-rate Planning Architecture}
The canonical CLF-CBF-QP framework optimizes the control input only one step ahead. Without a look-ahead horizon, the robot may become trapped in local minima and stall in non-convex environments. Motivated by this limitation, we propose a hierarchical multi-rate architecture, termed MILP-MPC-CBF, that combines long-horizon planning with geometry-aware safety filtering. Specifically, the proposed architecture consists of:

\begin{itemize}
  \item a \textbf{low-rate MILP-MPC planner} that generates a long-horizon nominal control sequence using a {\it point-robot} abstraction for computational tractability; and
  \item a \textbf{high-rate Minkowski-CBF safety filter} that tracks the nominal command while enforcing collision avoidance with respect to the {\it full robot geometry}.
\end{itemize}

\subsection{High-level MPC planner via MILP}
At each planning step $k$, the MILP-MPC  planner solves a finite-horizon optimal control problem (FHOCP) over a prediction horizon of $N_p$ steps with the sampling interval $\Delta t_p$, resulting in a prediction horizon of duration $T_p = N_p \Delta t_p$. The FHOCP is initialized at the current state, i.e.,  $\bm{x}_0=\bm{x}_k$. The decision variables consist of both continuous and binary variables. Let $\mathbf{X}$ and $\mathbf{U}$ denote the predicted state and control trajectories, respectively:
\begin{align*}
    \mathbf{X} &:= \{\bm{x}_i\}_{i=0}^{N_p}
    = [\bm{x}_0,\bm{x}_{1},\ldots,\bm{x}_{N_p}]
    \in \mathbb{R}^{n\times (N_p+1)},\\[0.2em]
    \mathbf{U} &:= \{\bm{u}_i\}_{i=0}^{N_p-1}
    = [\bm{u}_0,\bm{u}_1,\ldots,\bm{u}_{N_p-1}]
    \in \mathbb{R}^{m\times N_p}.
\end{align*}
To encode polytopic obstacle avoidance over the prediction horizon $N_p$, let $\ell_{\mathcal{O}_j}$ denote the number of half-space inequalities defining obstacle $\mathcal{O}_j,~\forall j\in\{1,\dots,N_{\mathcal O}\}$. For each obstacle $j$ at each prediction step $i$, we introduce a binary vector $\bm{t}_{i}^j$:
\begin{align}
    \bm{t}_{i}^j :=
    \begin{bmatrix}
        t_{i,1}^j,\ldots,t_{i,\ell_{\mathcal{O}_j}}^j
    \end{bmatrix}^{\top}
    \in \{0,1\}^{\ell_{\mathcal{O}_j}}.
\end{align}
These vectors are then collected in $\mathbf{T}$:
\begin{align*}
    &\mathbf{T} := \{ \bm{t}_{i}^j \}_{i=0,\ldots
    ,N_p}^{j=1,\ldots
    ,N_{\mathcal O}}.
\end{align*}
The vector $\bm{t}_{i}^j$
is used to encode the disjunctive half-space constraints associated with obstacle $\mathcal{O}_j$ in $\mathcal{W}$-space. 
For each prediction step $i=0,\ldots,N_p$ and each obstacle
$j=1,\ldots,N_{\mathcal O}$, the corresponding obstacle-avoidance constraints are imposed using the following ``Big-M'' formulation:
\begin{subequations}\label{eq:obs_bigM}
\begin{align}
-A_{\mathcal O_j,r}\bm p_i
&\le -b_{\mathcal O_j,r}-\varepsilon_{\mathrm{obs}} + M t_{i,r}^j,
\label{eq:obs_bigM_a}\\
&\qquad r=1,\ldots,\ell_{\mathcal{O}_j}, \notag\\
\sum_{r=1}^{\ell_{\mathcal{O}_j}} t_{i,r}^j
&\le \ell_{\mathcal{O}_j}-1,
\label{eq:obs_bigM_b}
\end{align}
\end{subequations}
where $\bm{p}_i$ is the position at each step $i$, $M$ is a sufficiently large constant and $\varepsilon_{\mathrm{obs}}>0$ is a user-defined safety margin. Constraint \eqref{eq:obs_bigM_b} ensures that at least one entry of $\bm{t}_i^j$ is zero, so at least one corresponding inequality in \eqref{eq:obs_bigM_a} remains active, which enforces $\bm{p}_i$ to lie outside of the obstacle $\mathcal{O}_j$. The resulting MILP-MPC FHOCP is formulated as
\begin{subequations}\label{eq: MILP-MPC}
\begin{align}
    \min_{\mathbf{X},\mathbf{U},\mathbf{T}} &\sum_{i=0}^{N_p-1}\left(\|\bm{u}_i\|_1 + \beta \|\bm{x}_i-\bm{x}_g\|_1\right) + \alpha \|\bm{x}_{N_p}-\bm{x}_g\|_1 \label{eq:mpc_obj}
    \\
    \text{s.t.}\quad &\bm{x}_0 = \bm{x}_k \label{eq:mpc_init} 
    \\
    &\bm{x}_{i+1} = A\bm{x}_i + B\bm{u}_i,\quad i=0,\ldots,N_p-1, \label{eq:mpc_dyn}
    \\
    &\bm{x}_{\min} \le \bm{x}_i \le \bm{x}_{\max}, \quad i=0,\ldots,N_p, \label{eq:mpc_state_bounds}
    \\
    &\bm{u}_{\min} \le \bm{u}_i \le \bm{u}_{\max}, \quad i=0,\ldots,N_p-1. \label{eq:mpc_input_bounds}\\
    &\text{obstacle-avoidance constraints in \eqref{eq:obs_bigM_a}--\eqref{eq:obs_bigM_b}},\notag
    \end{align}
\end{subequations}
where $\alpha$ and $\beta$ are positive user-defined weighting parameters that penalize the terminal and running deviations from the goal, respectively. Solving \eqref{eq: MILP-MPC} at each planning step $k$ yields an optimal control sequence $\mathbf{U}^*=[\bm{u}_0^*,\bm{u}_1^*,\ldots,\bm{u}^*_{N_p-1}]$. As in standard receding-horizon MPC, only the first input $\bm{u}_0^*$ is extracted. In our framework, this input is interpreted as the nominal guidance command, denoted by $\bm{u}_{\mathrm{mpc}}:=\bm{u}_0^*$, and passed to the low-level Minkowski-CBF safety filter (detailed in Sec.~\ref{sec: M-cbf}) for geometry-aware collision avoidance. Specifically, the low-level safety filter synthesizes the control input by solving the following min-norm CBF-QP, which minimally modifies $\bm{u}_{\mathrm{mpc}}$ while enforcing the input and safety constraints \eqref{eq: def_hocbf_constr}:
\begin{align}\begin{split}
    \bm{u}^*(\bm{x})&=\arg\min_{\bm{u}}\|\bm{u}-\bm{u}_{\mathrm{mpc}}\|^2_2
    \quad \text{s.t.}\,\,\bm{u}\in\mathcal{U}\,\,\text{and}\,\,\eqref{eq: def_hocbf_constr}.
    \label{eq: cbf-min-norm}
\end{split}\end{align}
The FHOCP \eqref{eq: MILP-MPC} is then re-solved at the next planning step using the updated state. An overview of the proposed framework is summarized in Algorithm.~\ref{alg1}.

\begin{algorithm}[h]
\color{black}
\caption{Hierarchical, multi-rate MILP-MPC-CBF}
\begin{algorithmic}[1]
\STATE Initialize $\bm{x} \leftarrow \bm{x}_0$, $t\leftarrow 0$, $k\leftarrow 0$, $\bm{u}_{\mathrm{mpc}}\leftarrow 0$.
\STATE Set $N_{\mathrm{MPC}}\leftarrow round(\Delta t_p/\Delta t)$.
\WHILE{$t \le t_{\mathrm{final}}$}
  \IF{$k = 0$ or $\bmod(k, N_{\mathrm{MPC}}) = 0$}
    \STATE Solve MILP-MPC \eqref{eq: MILP-MPC} with horizon $N_p$ and planner sampling step $\Delta t_{\mathrm{p}}$
    \IF{MILP feasible}
      \STATE $\bm{u}_{\mathrm{mpc}} \leftarrow \bm{u}_0^\star$.
    \ELSE
      \STATE $\bm{u}_{\mathrm{mpc}} \leftarrow$ greedy direction to goal
    \ENDIF
  \ENDIF
  \STATE Solve the min-dist QP \eqref{eq: dist_c} and compute the CBF value and gradient using the minimizer $\bm{z}^*$ and \eqref{eq: cbf_grad}.
  \IF{single-integrator}
     \STATE Construct the CBF constraint \eqref{eq: cbf_constr}.
  \ELSIF{double-integrator}
    \STATE Compute the Hessian using \eqref{eq: hessian_cases} in Thm.~\ref{thm: cbf_hessian}.
   \STATE Construct the HOCBF constraint \eqref{eq: hocbf_constr}.
   \ENDIF
  \STATE $\bm{u}^*\leftarrow$ Solve the min-norm CBF-QP \eqref{eq: cbf-min-norm} with reference $\bm{u}_{\mathrm{mpc}}$ and the corresponding CBF/HOCBF constraint.
  \STATE $\bm{x}\leftarrow$ Propagate the system dynamics using $\bm{u}^*$.
  \STATE $t \leftarrow t + \Delta t$, $k \leftarrow k+1$.
\ENDWHILE

\end{algorithmic}\label{alg1}
\end{algorithm}


\subsection{Low-level safety filter via Minkowski-CBF} \label{sec: M-cbf}
We first briefly review the Minkowski-CBF introduced in  \cite{Chen.etal.CDC25}, which is defined through the CO associated with a robot-obstacle pair in \mdspace. While this CBF admits both minimum-distance and penetration-depth cases, we focus here on the {\it minimum-distance} case, i.e., the collision-free case $\mathcal{R}\cap\mathcal{O}=\emptyset$ (equivalently $\bm{0}\notin\mathcal{O}^c$), since the robot is initialized in the safe set and the CBF guarantees forward invariance. In this case, the signed distance (sd) reduces to the minimum distance (dist), which can be computed via the following QP (referred to as a min-dist QP):
\begin{align}\begin{split}
   &\min\,\,\|\bm{z}\|^2_2\quad\text{s.t.}\,\,\,\,\mathbf{A}^c(\bm{x}) \bm{z}\leq \mathbf{b}^c(\bm{x}),\\
   &\,\text{with}\,\,\, 
   \text{dist}(\mathcal{R}(\bm{x}),\mathcal{O})=\text{dist}(\bm{0},\mathcal{O}^c(\bm{x})) =\|\bm{z}^*(\bm{x})\|_2,
   \label{eq: dist_c}
\end{split}\end{align}
where the minimizer $\bm{z}^*=\bm{z}^*(\bm{x})\in\mathbb{R}^d$, referred to as a critical point, is the projection of the origin $\bm{0}$ onto $\mathcal{O}^c$ in \mdspace. The Minkowski-CBF is 
\begin{align}\begin{split}   
    &h(\bm{x})=\text{dist}(\mathcal{R}(\bm{x}),\mathcal{O})-d_\text{safe}=\|\bm{z}^*(\bm{x})\|_2-d_\text{safe},
    \label{eq: min-dist-cbf}
\end{split}
\end{align}
where $d_\text{safe}>0$ is the user-defined safety margin. 
Although we retain $\mathbf{A}^c(\bm{x})$ and $\mathbf{b}^c(\bm{x})$ for generality, the CO representation depends directly only on the robot position $\bm{p}\in\mathbb{R}^d$ and orientation. In the pure-translation setting considered here, the orientation is fixed, so this dependence reduces to the position component $\bm{p}$ of the state $\bm{x}$.

For the QP in \eqref{eq: dist_c}, the objective function is $f_0(\bm{z})=\bm{z}^\top\bm{z}$, and the corresponding Lagrangian is
\begin{align}
    \mathcal{L}(\bm{z},\boldsymbol{\lambda},\bm{x})=f_o(\bm{z})+ \boldsymbol{\lambda}^\top \left(\mathbf{A}^c(\bm{x})\bm{z}-\mathbf{b}^c(\bm{x})\right),\label{eq: Lag}
\end{align}
where $\boldsymbol{\lambda}\in\mathbb{R}^{\ell_c}$ is the dual variable. The following theorem gives a closed-form expression for the gradient of the CBF with respect to the position $\bm{p}$.

\begin{thm}[CBF Gradient with respect to Position]
    \normalfont
    For $h(\bm{x})>0$, the gradient of the CBF $h$ in \eqref{eq: min-dist-cbf} with respect to the position $\bm{p}$ is:
    \begin{align}
        \frac{\partial h}{\partial \bm{p}} =-\frac{(\bm{z}^*)^\top}{\|\bm{z}^*\|_2}, \label{eq: cbf_grad}
    \end{align}
    i.e., the unit vector pointing from the critical point $\bm{z}^*$ toward the origin. \label{thm: cbf_grad}
\end{thm}

\begin{proof}
    For $h(\bm{x})>0$, we have $h=\|\bm{z}^*\|_2-d_\text{safe}=\sqrt{f_0(\bm{z}^*)}-d_\text{safe}$.
    By complementary slackness, the inactive dual variables are zero, so the stationarity condition reduces on the active set $\mathcal{I}$ to 
    $2\bm{z}^*+\mathbf{A}^c_\mathcal{I}(\bm{p})^\top \boldsymbol{\lambda}_\mathcal{I}^*=0$,
    where the subscript $\mathcal{I}$ denotes the active constraint set. Applying the chain rule together with Theorem~\ref{Thm:dfdx}, using the Lagrangian in \eqref{eq: Lag}, yields:
    \begin{align}
        \frac{\partial h}{\partial \bm{p}} &=\frac{\partial h}{\partial f_0(\bm{z}^*)}\frac{\partial f_0(\bm{z}^*)}{\partial \bm{p}}=\frac{1}{2\|\bm{z}^*\|_2}\frac{\partial\mathcal{L}}{\partial \bm{p}}\bigg\rvert_{(\bm{z}^*,\boldsymbol{\lambda}_\mathcal{I}^*)}  \label{eq: thm+chain-rule}
    \end{align}
    In the pure-translation setting, we have $\partial_{\bm{p}}\mathbf{A}_\mathcal{I}^c=\bm{0},~\partial_{\bm{p}}\mathbf{b}_\mathcal{I}^c=-\mathbf{A}_\mathcal{I}^c$ \cite{Chen.etal.CDC25}. Hence, 
    \begin{align}
        \frac{\partial\mathcal{L}}{\partial \bm{p}}\bigg\rvert_{(\bm{z}^*,\boldsymbol{\lambda}^*_\mathcal{I})}= (\boldsymbol{\lambda}_\mathcal{I}^*)^\top
        \left[ \left( \partial_{\bm{p}}\mathbf{A}_\mathcal{I}^c\right)\bm{z}^*-\partial_{\bm{p}}\mathbf{b}_\mathcal{I}^c\right]
        = (\boldsymbol{\lambda}_\mathcal{I}^*)^\top\Aci \label{eq: dLdp}
    \end{align}
    Using the stationarity condition $(\boldsymbol{\lambda}_\mathcal{I}^*)^\top\mathbf{A}_\mathcal{I}^c=-2(\bm{z}^*)^\top$ and substitutes it into \eqref{eq: thm+chain-rule} and \eqref{eq: dLdp}, we obtain:
    \begin{align}
        \frac{\partial h}{\partial \bm{p}} =\frac{1}{2\|\bm{z}^*\|_2}\left(-2(\bm{z}^*)^\top\right)= - \frac{(\bm{z}^*)^\top}{\|\bm{z}^*\|_2},
    \end{align}
    which completes the proof.\qedhere
    
\end{proof}

\begin{remark}
In our prior work \cite{Chen.etal.CDC25}, the gradient was derived from the sensitivity of the optimal solution \(z^*(x)\) via the implicit function theorem (IFT) \cite{Amos.Kolter.ICML17}. Here, we instead use Theorem.~\ref{Thm:dfdx} as a tool to derive the closed-form gradient expression. These two derivations are essentially equivalent, but the present derivation simplifies the proof.
\end{remark}

\begin{remark}
    In the minimum-distance case, the critical point $\bm{z}^*$ is the projection of the origin onto the convex CO, and is therefore unique. Moreover, the distance to a closed convex set is continuously differentiable outside the set \cite{Poliquin.etal.AMS00}. Hence, in the pure-translation setting, the CBF $h(\bm{x})$ is continuously differentiable.
\end{remark}
 We denote the transpose of the CBF gradient with respect to the position \eqref{eq: cbf_grad} as the unit normal vector: 
\begin{align}
    \mathbf{n}:=\left(\frac{\partial h}{\partial \bm{p}}\right)^\top=-\frac{\bm{z}^*}{\|\bm{z}^*\|_2}\in\mathbb{R}^d
    \label{eq: norm_vec}
\end{align}
The resulting first-order CBF constraint with the control-affine system \eqref{eq: ctrl-affine} is as follows:
\begin{align}
    \mathbf{n}^\top f(\bm{x})+\mathbf{n}^\top g(\bm{x})\bm{u}+kh(\bm{p})\geq0.
    \label{eq: cbf_constr}
\end{align}


\subsection{Extension to HOCBF for double-integrator dynamics}
We now extend the Minkowski-CBF construction to HOCBF. Since the safety function depends only on the position component of the state, it has relative degree two with respect to the control input. Let the state be $\bm{x}=[\bm{p}^\top,\bm{v}^\top]^\top$, where $\bm{v}\in\mathbb{R}^d$ denotes the velocity component. Following Def.~\ref{def: hocbf}, we construct a HOCBF by defining $\psi_0(\bm{x}):=h(\bm{x})=\|\bm{z}^*\|_2-d_\text{safe}$, $\alpha_1(\psi_0(\bm{x})):=k_1\psi_0(\bm{x})$, and $\alpha_2(\psi_1(\bm{x})):=k_2\psi_1(\bm{x})$, where $k_1,k_2>0$ are hyper-parameters. Hereafter, we restrict ourselves to 2D $\mathcal{W}$-space ($d=2$).
The resulting HOCBF constraint \eqref{eq: def_hocbf_constr} in this case is:
\begin{align}\begin{split}
    &\bm{v}^\top H\bm{v}+\mathbf{n}^\top\bm{u}+(k_1+k_2)\mathbf{n}^\top\bm{v}+k_1k_2h(\bm{p})\geq 0,\label{eq: hocbf_constr}
\end{split}\end{align}
where $H:=\partial^2_{\bm{p}}h\in\mathbb{R}^{2\times 2}$ is the Hessian matrix of the CBF with respect to position. Computing $H$ is generally intractable; however, our Minkowski-based formulation admits an analytical expression via the KKT sensitivity of \eqref{eq: dist_c}, as presented in Theorem.~\ref{thm: cbf_hessian}.

\begin{thm}[CBF Hessian with respect to Position] \label{thm: cbf_hessian}
    \normalfont 
    For $h(\bm{x})>0$, the Hessian of the CBF $h$ in \eqref{eq: min-dist-cbf} with respect to the position $\bm{p}\in\mathbb{R}^2$ is given by:
    \begin{equation}
        H = \begin{cases} 
        \frac{1}{\|\bm{z}^*\|_2} (\mathbf{I} - \mathbf{n}\mathbf{n}^\top), & \text{if } \bm{z}^* \text{ lies at a vertex,} \\
        \mathbf{0}, & \text{if } \bm{z}^* \text{ lies on an edge,}
        \end{cases}
        \label{eq: hessian_cases}
    \end{equation}
    where $\mathbf{I}\in\mathbb{R}^{2\times 2}$ is the identity matrix and $\mathbf{n}$ is the unit vector defined in \eqref{eq: norm_vec}.
\end{thm}
\begin{proof}
    By definition, the Hessian matrix of $h$ with respect to $\bm{p}$ is $
        H=\partial_{\bm{p}}(\partial_{\bm{p}}h)=\partial_{\bm{p}}(\mathbf{n}^\top)$.
    Using \eqref{eq: norm_vec}, the quotient rule and
    \begin{align*}
        \frac{\partial\|\bm{z}^*\|_2}{\partial \bm{p}}=\frac{(\bm{z}^*)^\top}{\|\bm{z}^*\|_2}\frac{\partial\bm{z}^*}{\partial \bm{p}}=-\mathbf{n}^\top\frac{\partial\bm{z}^*}{\partial \bm{p}},
    \end{align*}
    we obtain the following fundamental expression:
    \begin{align}
        H = \frac{1}{\|\bm{z}^*\|_2} (\mathbf{I} - \mathbf{n}\mathbf{n}^\top) \left( -\frac{\partial \bm{z}^*}{\partial \bm{p}} \right). \label{eq: hessian_g}
    \end{align} 
    To evaluate the last term of $\partial \bm{z}^*/\partial \bm{p}$, we perform a local sensitivity analysis on the KKT conditions of the min-dist QP in \eqref{eq: dist_c} using the IFT, restricted to regions where the active feature remains unchanged. Let $\bm{\xi}:=(\bm{z},\boldsymbol{\lambda})$ and $\Gamma$ denote the KKT system in \eqref{eq: KKT_cond} with $\mathbf{Q}=2\mathbf{I},~\mathbf{p}=\bm{0},~\mathbf{G}=\mathbf{A}^c,~\mathbf{h}=\mathbf{b}^c$. Considering only active constraints and recalling $\partial_{\bm{p}}\mathbf{A}_\mathcal{I}^c=\bm{0},~\partial_{\bm{p}}\mathbf{b}_\mathcal{I}^c=-\mathbf{A}_\mathcal{I}^c$, we obtain the sensitivity equation:
    \begin{align*}
    &\partial_{\bm{p}}(\bm{\xi}^*) = -\left[ \partial_{\bm{\xi}}\Gamma(\bm{\xi}^*)\right]^{-1} \partial_{\bm{p}}\Gamma(\bm{\xi}^*),\\[0.2em]
    &\partial_{\bm{\xi}}\Gamma =
    \begin{bmatrix}
    2\mathbf{I} & \left(\Aci\right)^{\top} \\D(\lambdai^*)\Aci & \bm{0}
    \end{bmatrix},\,\,
    \partial_{\bm{p}}\Gamma =
    \begin{bmatrix}
    \bm{0} \\D(\lambdai^*)\left( \Aci \right)
    \end{bmatrix},
\end{align*}
The inverse matrix $\left[ \partial_{\bm{\xi}}\Gamma(\bm{\xi}^*)\right]^{-1}$ is computed using the Schur complement. Since the upper block of $\partial_{\bm{p}}\Gamma$ is $\bm{0}$, we only need the top-right block of $\left[ \partial_{\bm{\xi}}\Gamma(\bm{\xi}^*)\right]^{-1}$ for $\partial\bm{z}^*/\partial \bm{p}$, which analytically evaluates to $(\Aci)^\top\left[\Aci(\Aci)^\top\right]^{-1}\left[ D(\lambdai^*)\right]^{-1}$. Multiplying this block by the lower component of $\partial_{\bm{p}}\Gamma$, the diagonal matrices exactly cancel out ($\left[ D(\lambdai^*)\right]^{-1}D(\lambdai^*)=\mathbf{I}$) under strict complementarity, yielding the sensitivity:
\begin{align}
    \frac{\partial \bm{z}^*}{\partial \bm{p}}=-\left(\Aci\right)^\top\left[\Aci\left(\Aci\right)^\top\right]^{-1}\Aci \label{eq: dz*dp}
\end{align}

\begin{itemize}[leftmargin=*]
    \item \textbf{Case~1~(Vertex):} When $\bm{z}^*$ lies at a vertex of the CO, the active set consists of two linearly independent constraints. Consequently, the active constraint matrix $\Aci\in\mathbb{R}^{2\times 2}$ is an invertible square matrix, so the sensitivity formula in \eqref{eq: dz*dp} simplifies to: $\partial \bm{z}^*/\partial \bm{p}=-\mathbf{I}$.
    Substituting this into \eqref{eq: hessian_g} yields $H = (\mathbf{I} - \mathbf{n}\mathbf{n}^\top)/\|\bm{z}^*\|_2$.
    \item \textbf{Case~2~(Edge):} When $z^*$ lies on an edge of the CO, the active set reduces to a single constraint, denoted as $\mathbf{a}:=\left(\Aci\right)^\top\in\mathbb{R}^2$. Substituting this into \eqref{eq: dz*dp} yields $-\mathbf{a}\left(\mathbf{a}^\top\mathbf{a}\right)^{-1}\mathbf{a}^\top=-\mathbf{a}\mathbf{a}^\top/\|\mathbf{a}\|_2^2$.Based on the KKT stationarity condition $2\bm{z}^*+\lambdai^*\mathbf{a}=\bm{0}$, the vector $\mathbf{a}$ is collinear with $\bm{z}^*$. Since $\mathbf{n}$ defined in \eqref{eq: norm_vec} is also a unit vector collinear with $\bm{z}^*$, the term $\mathbf{a}\mathbf{a}^\top/\|\mathbf{a}\|_2^2$ is equivalent to $\mathbf{n}\mathbf{n}^\top$. Thus, $\partial \bm{z}^*/\partial \bm{p}=-\mathbf{n}\mathbf{n}^\top$. Substituting this back into \eqref{eq: hessian_g} yields $H =  (\mathbf{I} - \mathbf{n}\mathbf{n}^\top)(\mathbf{n}\mathbf{n}^\top)/\|\bm{z}^*\|_2 = \mathbf{0}$, since $\mathbf{n}^\top\mathbf{n} = 1$. This completes the proof. \qedhere
\end{itemize}
\end{proof}
\begin{remark}
    Theorem~\ref{thm: cbf_hessian} applies away from active-set switching boundaries. At such boundaries, the Hessian may fail to exist. Since these boundaries form a measure-zero set in $\mathbb{R}^2$, the Hessian exists almost everywhere.
\end{remark}
Intuitively, the Hessian captures the curvature of the distance field, yielding a zero matrix for straight level sets when $\bm{z}^*$ is on an edge, and a non-zero matrix for concentric circular level sets when $\bm{z}^*$ is at a vertex. With the above closed-form Hessian \eqref{eq: hessian_cases}, the HOCBF constraint \eqref{eq: hocbf_constr} can be obtained within each local active-feature region. We next validate the resulting hierarchical MILP-MPC-CBF framework in simulation.

\section{Simulation Results}
Throughout the simulations, we consider a triangular robot navigating toward a goal point $\bm{x}_g$ 
while strictly avoiding polygonal obstacles. We evaluate our proposed framework in two distinct scenarios: a complex maze environment to demonstrate long-horizon guidance, and a specifically designed U-shaped environment for comparative evaluation against baseline methods, i.e., MILP-MPC only planner and purely reactive CLF-CBF-QP. (See video at \textnormal{\url{https://youtu.be/-cm4cQ_WXDY}} .)

All simulations are conducted in MATLAB on a PC with an Intel i7-8700 3.2 GHz 6-core CPU and 32 GB RAM. The MILP-MPC FHOCPs 
\eqref{eq: MILP-MPC} are solved using \texttt{intlinprog}, while the min-dist \eqref{eq: dist_c} and CBF-filtering QPs \eqref{eq: cbf-min-norm} via \texttt{quadprog}. The closed-loop system and the low-level QPs are both updated at 100 Hz. The high-level planner uses a discretization step of $\Delta t_p=0.2\,\mathrm{s}$ and a prediction horizon of $N_p=10$, yielding a $T_p=2\,\mathrm{s}$ look-ahead. The algorithm parameters are set as follows: $M=20,~ \varepsilon_{\mathrm{obs}}=0.01,~ \alpha=20,~\beta=0.08$ for MILP-MPC; $k=3$ for the CBF; $k_1=2,~k_2=10$ for the HOCBF, $d_{\mathrm{safe}}=0$ for safety margin; and control bounds $\bm{u}_\mathrm{max}=-\bm{u}_\mathrm{min}=5$.
\subsection{System Dynamics}
\begin{itemize}[leftmargin=*]
    \item \textbf{Case~1~(Single-integrator):} The system state is the position $\bm{x} = \bm{p} = [x,y]^\top \in \mathbb{R}^2$, with dynamics governed by $\dot{\bm{x}} = \bm{u} = [u_1, u_2]^\top$, where $\bm{u}$ represents the linear velocity control inputs. These dynamics can be directly expressed in the control-affine form of \eqref{eq: ctrl-affine} by defining the drift vector as $f(\bm{x})=\bm{0}_{2\times 1}$ and the control matrix as $g(\bm{x})=\mathbf{I}_2$. 
    
    \item \textbf{Case~2~(Double-integrator):} The state space is expanded to include velocity, i.e., $\bm{x} = [\bm{p}^\top, \bm{v}^\top]^\top \in \mathbb{R}^4$. The corresponding dynamics are $\dot{\bm{p}} = \bm{v}$ and $\dot{\bm{v}} = \bm{u}$, where the control input $\bm{u} \in \mathbb{R}^2$ dictates the planar acceleration. In the control-affine form of \eqref{eq: ctrl-affine}, this translates to $f(\bm{x}) = [\bm{v}^\top, 0, 0]^\top$ and $g(\bm{x}) = [\bm{0}_{2\times 2}; \mathbf{I}_2]$.
\end{itemize}

\subsection{Long-Horizon Predictive Guidance}
The primary motivation for integrating MILP-MPC with CBFs is to leverage look-ahead predictions to improve global navigational capabilities, allowing the system to navigate without requiring predefined waypoints or a reference path from an external global planner. To maintain computational tractability, our framework adopts a hierarchical architecture: the high-level MILP-MPC executes at a low rate of 5 Hz using a point-mass model to generate a nominal reference path. The low-level Minkowski-CBF operates at a high rate of 100 Hz to account for the robot's full geometric shape, acting as a real-time safety filter that minimally adjusts the nominal control input to guarantee strict safety.
\begin{figure}[htb!]
	\centering
    \includegraphics[width=\linewidth]{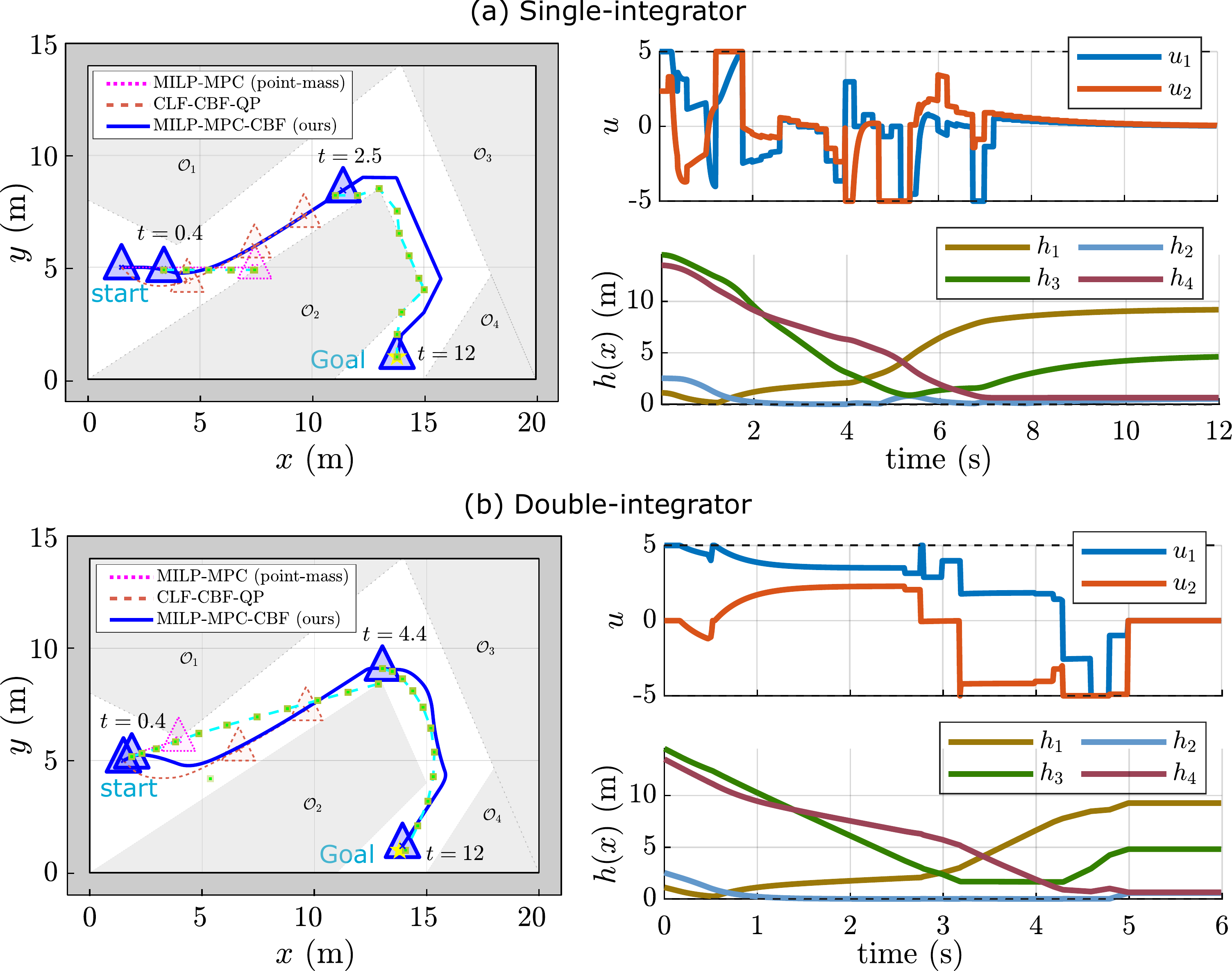}
	\caption{Long-horizon nominal predictions (dashed cyan) generated by the high-level MILP-MPC planner under (a) single- (\textbf{Top}) and (b) double-integrator (\textbf{Bottom}) dynamics, navigating the robot to the goal. \textbf{Left: } The actual closed-loop trajectory (solid blue) tracks the MPC predicted states (brown squares) as the robot (blue) avoids obstacles (gray). \textbf{Right: } Control inputs $\bm{u}$ and CBF values, verifying that strict safety $h(\bm{x})\geq0$ is guaranteed with respect to the real robot's geometry over time.
    }\label{fig: maze}
\end{figure}

To demonstrate this predictive capability, we consider the maze navigation scenario in Fig.~\ref{fig: maze}, initializing the robot at $\bm{x}_s=(1.5,5)$ to reach the goal $\bm{x}_g=(13.8,1)$. Due to the $T_p=2\,\mathrm{s}$ look-ahead horizon, the MILP-MPC anticipates the protruding vertex of the obstacle 2. It generates a nominal trajectory that guides the robot upward---temporarily moving away from the goal---to safely bypass the corner. In contrast, a purely reactive CLF-CBF-QP  ($V(\bm{x})=||\bm{x}-\bm{x}_g||^2_2$) controller greedily minimizes the distance to the goal based solely on instantaneous local geometry. This myopic approach inevitably traps the robot in a local minimum along the slanted edge, as shown in Fig.~\ref{fig: teaser}.

\begin{figure*}[t!]
    \includegraphics[width=0.98\textwidth]{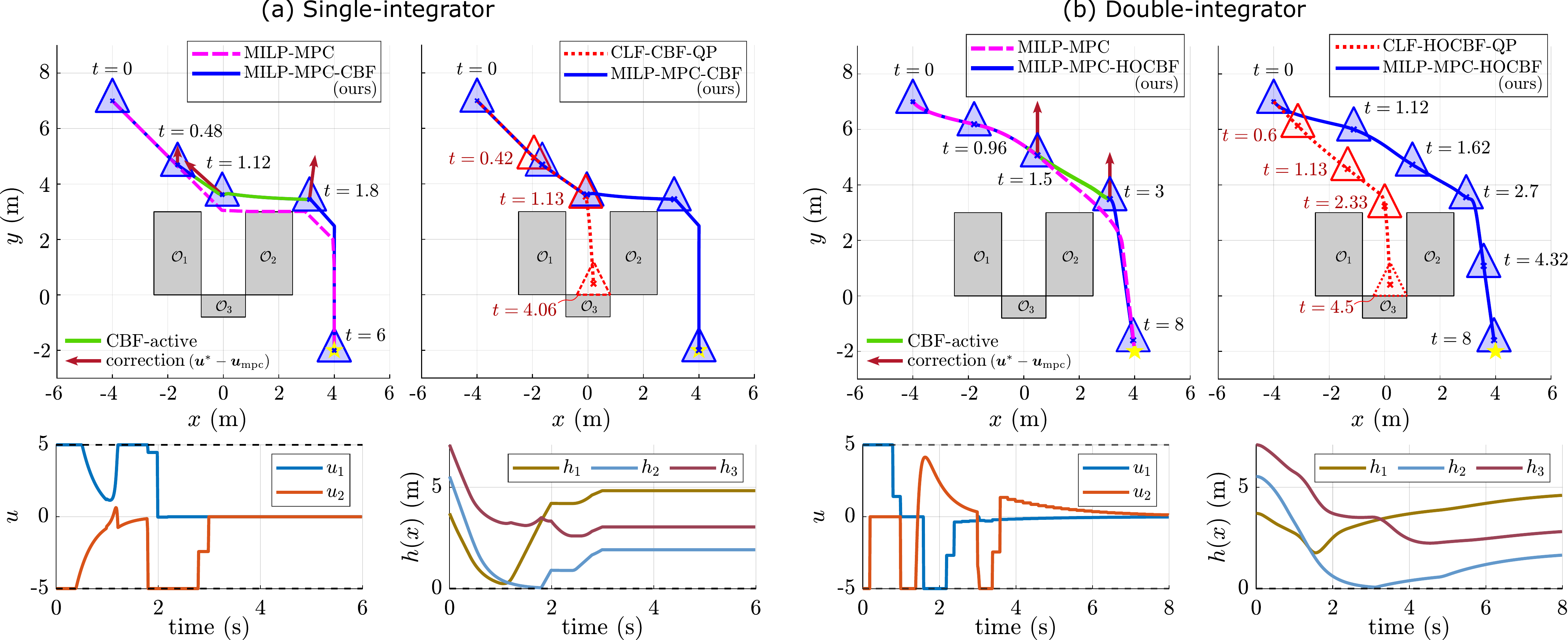}
    \caption{\textbf{Top-row: } Comparison of the proposed MILP-MPC-CBF framework in a U-shaped environment under single- (\textbf{Left} column) and double-integrator (\textbf{Right} column) dynamics. In each case, the left compares the proposed framework with MILP-MPC only, where the green segment indicates the interval during which the geometry-aware Minkowski-(HO)CBF safety filter is active, while the right compares it with the reactive CLF-(HO)CBF-QP. \textbf{Bottom-row: } Time histories of control inputs $\bm{u}$ and CBFs $h$ for MILP-MPC-CBF, showing that the full-geometry robot remains in the safe set throughout the navigation.
    }\label{fig: u-trap}
\end{figure*}

\subsection{Comparative Evaluation}
To highlight the usefulness of our hierarchical design,  we evaluate the frameworks in a U-shaped trap, see Fig.~\ref{fig: u-trap}. We first demonstrate the crucial role of the low-level geometry-aware safety filter by comparing it against a standalone MILP-MPC planner (dashed magenta). Because the high-level planner relies on a point-mass model and enforces safety only at discrete time steps, its optimal trajectory tightly traces the obstacle boundaries. As illustrated in the left subfigures of the top row of Fig.~\ref{fig: u-trap}, our Minkowski-CBF activates whenever the nominal control becomes unsafe (green segments). It applies minimally invasive real-time corrections (dark red arrows) to repel the full-geometry robot from the obstacle, ensuring strict safety while preserving the long-horizon guidance dictated by the MILP-MPC.

Next, as shown in the right subfigures of the top row of Fig.~\ref{fig: u-trap}, we compare our MILP-MPC-CBF (solid blue) against the purely reactive CLF-CBF-QP baseline (dotted red). Lacking a predictive horizon, the baseline is driven directly into the concavity and deadlocks at zero velocity due to conflicting goal-seeking and safety objectives. In contrast, our approach proactively bypasses this trap by tracking the long-horizon nominal path.


\subsection{Discussion and Limitations}
While our MILP-MPC-CBF framework has demonstrated improved navigational performance compared to purely reactive methods, certain limitations remain. Given a sufficiently 
long prediction horizon, it can effectively resolve topology-induced local minima. However, the hierarchical design introduces a model mismatch between the high-level point-mass planner and the low-level full-geometry safety filter. This mismatch becomes non-negligible when the robot's actual shape deviates significantly from the simplified geometry models. In such cases, the safety filter will override the nominal guidance, causing the system to be in a geometry-induced local minimum. Resolving geometry-induced deadlock caused by planner-filter mismatch remains an open problem. 


\section{Conclusions and Future Work}
We presented a hierarchical navigation framework for geometry-aware safe navigation of polytopic robots in complex polytopic environments. By deriving the closed-form gradient and Hessian of the Minkowski-CBF, we enabled a HOCBF formulation for double-integrator dynamics. Simulations demonstrate that our approach effectively resolves topology-induced local minima of purely reactive CLF-CBF-QP controllers and compensates for the actual physical volume neglected by point-mass abstractions to guarantee safety for the robot's full geometry. Future work will extend this framework to orientation-varying robots, nonlinear dynamics, and incorporate predicted states into the CBF to systematically resolve geometry-induced deadlocks.
\section{Acknowledgment}
The authors used generative AI tools \cite{google_gemini31pro_2026,openai_chatgpt_2025} for grammar polishing, manuscript organization, and limited clarification of derivation details. All research ideas, technical content, and analysis were developed by the authors.


\bibliographystyle{IEEEtran}
\bibliography{references}

\end{document}